%% file: main.tex
\documentclass{article}

\usepackage{arxiv}

\usepackage[utf8]{inputenc}             
\usepackage[T1]{fontenc}                
\usepackage{hyperref}                   
\usepackage{url}                        
\usepackage{booktabs}                   
\usepackage{amsfonts}                   
\usepackage{nicefrac}                   
\usepackage{microtype}                  
\usepackage[bottom]{footmisc}
\usepackage{graphicx}
\usepackage{subfigure}
\usepackage{booktabs}                   
\usepackage{amsmath,amssymb,bbm}
\usepackage{hyperref}
\usepackage{amsthm}
\usepackage{verbatim}
\usepackage{algorithm}
\usepackage{algorithmic}
\usepackage{comment}
\usepackage[normalem]{ulem}

\def\~#1{\mathbb{#1}}
\def\*#1{\mathbf{#1}}
\def\@#1{\mathcal{#1}}

\newtheorem{claim}{Claim}[section]
\newtheorem{definition}{Definition}[section]
\newtheorem{theorem}{Theorem}[section]
\newtheorem{corollary}{Corollary}[section]
\newtheorem{lemma}{Lemma}[section]

\title{Efficient NTK using Dimensionality Reduction}

\author{
  Nir Ailon\\
  Technion\\
  \texttt{nailon@gmail.com}
   \And \And \And
   Supratim Shit \thanks{Corresponding author} \\
  Technion \\
  \texttt{supratim.shit@gmail.com} 
}

\begin{document}
\maketitle

\begin{abstract}
Recently, neural tangent kernel (NTK) has been used to explain the dynamics of learning parameters of neural networks, at the large width limit. Quantitative analyses of NTK give rise to network widths that are often impractical and incur high costs in time and energy in both training and deployment. Using a matrix factorization technique, we show how to obtain similar guarantees to those obtained by a prior analysis while reducing training and inference resource costs. The importance of our result further increases when the input points' data dimension is in the same order as the number of input points. More generally, our work suggests how to analyze large width networks in which dense linear layers are replaced with a low complexity factorization, thus reducing the heavy dependence on the large width.
\end{abstract}
 

\allowdisplaybreaks

\input{sections/introduction}
\input{sections/preliminary}
\input{sections/dimension}
\input{sections/neurons}
\input{sections/ack}

\bibliographystyle{plainyr}
\bibliography{reference}


\end{document}

%% file: sections/introduction.tex
\section{Introduction}
 Deep learning techniques have  overwhelmed the world of machine learning  in the last decade. Among them, one of the popular techniques is the deep neural network. Usually, in practice, these networks are overparameterized, and they proved to be extremely successful in various applications like image recognition \cite{krizhevsky2012imagenet,lecun1998gradient},  natural language understanding \cite{devlin2019bert, collobert2011natural}, speech recognition \cite{hinton2012deep} etc. Due to the non-convexity of the cost function and the complex dynamics of the training process, it is incredibly challenging to explain the reason for their success. 
 
 In a recent breakthrough, Jacot et. al., \cite{jacot2018neural} discovered that 
 the dynamics of learning the network parameters at the infinite width limit regime tend to be that of a kernelized linear regression setting, with 
 respect to a feature map known as the \emph{Neural Tangent Kernel} (NTK).
 Du et. al., \cite{du2018gradient} improved the original analysis by providing a sharper non-asymptotic bound for finite network width. 
 Further, using the NTK idea, Arora et. al., \cite{arora2019fine} showed why the training error goes to zero even when adding adversarial noise to the input data. They also showed why these overparameterized networks have benign overfitting behaviour. These results follow when the network width is very high, $O(\mbox{poly}(n))$, where $n$ is the number of input points. Unfortunately, in a regime where input points are in $\~R^{O(n)}$ dimensional space, the number of trainable parameters is infeasible. Therefore, training such a network would incur high time and energy costs. We refer the reader to \cite{strubell2019energy} to understand the adverse societal and environmental effects of such costs. 

 In another line of work, matrix factorization has been proven to be a useful tool in reducing costs of both training and deploying deep networks, e.g. \cite{khodak2020initialization, ailon2021sparse, dao2019learning}.  In this work, we provide a connection between matrix factorization ideas and NTK analysis, suggesting that the resources required to obtain NTK guarantees may not necessarily be as high as reported in previous work.  
 
 Following the work of \cite{du2018gradient, arora2019fine}, we study a network containing a linear transformation of the input, followed by a RELU activation function,
 followed by a scalar valued linear function, and follow the dynamic of this network with respect to an MSE loss function.
 The NTK matrix given rise to by this network is well understood following the work of \cite{du2018gradient, arora2019fine}. Here we replace the first
 linear layer with a factorization thereof, and show that by carefully choosing the factorization parameters we can approximately preserve
 the NTK structure and hence the network dynamics. 
 
 The factorization we choose, following ideas from \cite{ailon2021sparse} is a random
 dimensionality reduction matrix (Johnson-Lindenstrauss) followed by
 a trainable matrix which increases the dimension.  The first matrix is constant, and does not change after the (random initialization).
 Using well known properties of JL matrices, we show how the dynamics of the factorized network, which is much cheaper to execute, 
 mimic those of the original matrix.  

%% file: sections/preliminary.tex
\section{Related Work}
Neural tangent kernel (NTK) was first introduced by Jacot et al. \cite{jacot2018neural}. They showed the way toward a better understanding of the dynamics of the training phase of a deep network. These techniques are also known as lazy training. It has been studied for many network architectures such as CNNs, RNNs, and ResNets \cite{du2018gradient, arora2019exact, alemohammad2020recurrent}. In \cite{du2018gradient} author shows the required network size and a learning rate to achieve optimal training loss on a fully connected neural network using gradient descent. Given a distribution of input points, the \cite{arora2019exact} showed the required number of parameters to ensure the desired generalization loss. Often these networks end up in an overparameterized regime with the number of parameters being very large than the size of the input data. These results are interesting in their own right, but they are still impractical in various applications. 

Randomized and deterministic techniques are extensively used to improve scalability of various learning algorithms such as clustering \cite{huang2020coresets, jiang2021coresets, chhaya2022coresets, shitonline}, classification \cite{tukan2020coresets, munteanu2021oblivious, mai2021coresets}, regression \cite{chhaya2020streaming, chhaya2020coresets, tukan2022new} and deep learning \cite{mussay2019data, mirzasoleiman2020coresets}. For a fully connected neural network, \cite{baykal2018data} uses sampling techniques to prune a learned network by removing less important nodes in the network. Their sampling technique is based on a sensitivity framework. In \cite{han2021random}, the authors showed how sketching, or sampling techniques can be used to randomly select features in a fully connected neural network with the Relu activation function. Thereby reducing the overall network parameters,

Matrix factorization is an elegant way of representing low rank matrices succinctly. For a fixed and known matrix, there are many deterministic and randomized techniques for obtaining low rank approximations \cite{mahoney2009cur, liberty2013simple, woodruff2014sketching}. The most well knows method is by the Johnson-Lindenstrauss transformation \cite{johnson1984extensions}. It was originally designed to preserve distances (or inner products) between vectors. It is also a method for decomposing large matrices as a composition of (random) dimensionality reduction followed by a transformation in lower dimensional space. There are various versions of this transformation \cite{ailon2010faster, matouvsek2008variants, larsen2017optimality}. A faster version of the Johnson-Lindenstrauss transformation was introduced by Ailon et. al. \cite{ailon2006approximate} which has been extensively used in various learning algorithms \cite{drineas2011faster, drineas2012fast}.

\section{Preliminary}
A matrix is represented by a bold capital letter, e.g., $\*M$. The $i^{th}$ row and $j^{th}$ column of the matrix are represented by $\*M_{i}$ and $\*M^{j}$ respectively. The $(i,j)^{th}$ entry of the matrix $\*M$ is represented as $\*M_{i,j}$. The 2-norm of a vector $\*v$ and spectral norm for matrix $\*M$ is represented by $\|\*v\|$ and $\|\*M\|$ respectively. We represent the ReLU function by $\sigma(\cdot)$ which is defined as $\sigma(z) = \max\{0,z\}$ for $z \in \~R$. We use $\mathbbm{1}\{\mathcal{E}\}$ as an indicator variable, which is $1$ if the event $\mathcal{E}$ is true else $0$. 

Consider a two layer fully connected neural network which takes input from $\~R^{d}$ and returns a scalar output in $\~R$. The first layer has $m$ neurons, represented by a weight matrix $\*W \in \~R^{m \times d}$. We use the ReLU function as the activation function in the first layer. The second (latent) layer is denoted by $\*v \in \~R^{m}$. The norm of every input points to the network is assumed to be $1$. Given $n$ samples of input as $\{\*X_{i},\*y_{i}\}_{i=1}^{n}$. We represent the set $\{\*X_{i}\}$ by a matrix $\*X \in \~R^{n \times d}$ and $\{\*y_{i}\}$ by an $n$-dimensional vector $\*y \in \~R^{n}$. Based on some network parameters $\*W$, and $\*v$ for every $\*X_{i}$ we get $\*u_{i}$ as follows,

\begin{equation}{\label{eq:W}}
  \*u_{i} = \frac{1}{\sqrt{m}}\sum_{r \leq m}\*v_{r}\cdot\sigma(\*W_{r}\*X_{i}^{T}).
\end{equation}

We represent the set $\{\*u_{i}\}$ by an $n$-dimensional vector $\*u$. We learn $\*W$ and $\*v$ such that it minimizes $\|\*y - \*u\|^{2}$.

In this paper we use the properties of the Johnson-Lindenstrauss transformation, which has been stated below for completeness.
\begin{definition}[Johnson-Lindenstrauss \cite{johnson1984extensions}\label{def:JL}]
 Let $\*M \in \~R^{n \times d}$ represents $n$ points in $\~R^d$, $\delta \in (0,1)$. Let $\*S$ be a $d \times k$ random matrix such that $k = O(\log(n/\delta))$. For each $i \in [d]$ and $j \in [k]$, $\*S_{i,j} \sim \@N(0,1/\sqrt{d})$. Then for every pair of $i,j \in [n]$ we have $|\*M_{i}\*M_{j}^{T} - \*M_{i}\*S\*S^{T}\*M_{j}^{T}| \leq \|\*M_{i}\|\cdot\|\*M_{j}\|$ with probability at least $1-\delta$.
\end{definition}
The matrix $\*S$ is also known as the JL matrix. 

\subsection{Neural Tangent Kernel}
Recent work by Du et. al. \cite{du2018gradient} showed that in an overparameterized (number of neurons $\Omega(n^{6})$) neural network, the dynamics of the gradient descent for a small learning rate $O(\lambda_{0}n^{-2})$ behaves like a linear model. 
The linear model is captured by a kernel matrix called the neural tangent kernel (NTK). Given $\*X$ the NTK matrix $\*H^{\infty} \in \~R^{n \times n}$ is a gram matrix with the $(i,j)^{th}$ term defined as follows,

\begin{equation*}
  \*H_{i,j}^{\infty} = \~E_{\*w \sim \@{N}(\*0,\*I_{m})}\left[\*X_{i}\*X_{j}^{T}\mathbbm{1}\{\*w\*X_{i}^{T} \geq 0, \*w\*X_{j}^{T} \geq 0\}\right].
\end{equation*}

The main result in the paper reads as follows:
\begin{theorem}{\label{thm:du}}
 Assume $\lambda_0 = \lambda_{\min}(\*H^{\infty}) > 0$. For $\delta \in (0,1)$, number of neurons $m = \Omega(\frac{n^{6}}{\lambda_{0}^{4}\delta^{2}})$, set $\*W_{r} \sim \@{N}(\*0,\*I)$ for every $r \in [m]$ and we set $\eta = O(\frac{\lambda_{0}}{n^{2}})$ then during the gradient descent, we have the following with probability at least $1-\delta$ at every $t = 0, 1, \ldots$,
 \begin{equation}
  \|\*u(t) - \*y\|^{2} \leq \left(1-\frac{\eta\lambda_{0}}{2}\right)^{t}\|\*u(0) - \*y\|^{2}
 \end{equation} 
\end{theorem}

Now, assuming that we fix the second layer parameter, in order to ensure that $\|\*u(t) - \*y\|^{2} \leq \varepsilon \|\*u(0) - \*y\|^{2}$ for some small $\varepsilon \in (0,1)$, the gradient descent has to run for $t = \Omega(\frac{\log(\varepsilon)}{\log(1-(\lambda_{0}/n)^{2})})$ steps. The time taken to update $\*W(t+1)$ from $\*W(t)$ is $\Omega(nmd)$, i.e., $\Omega\left(\frac{n^{7}d}{\lambda_{0}^{4}\delta^{2}}\right)$. Now, in the regime where $d = \Theta(n)$, it will take $\Omega\left(\frac{n^{8}\log(\varepsilon)}{\lambda_{0}^{4}\delta^{2}\log(1-(\lambda_{0}/n)^{2})}\right)$ running time to achieve the above training loss. Notice the dependence on $m$ (number of neurons) and the dimension of the input vectors. 

In this paper, we apply dimensionality reduction techniques to reduce the effective number of neurons $m$ and the dimension of input points $d$ and thereby improving the overall running time. We present our results in two parts. First, we show how to reduce the dimension of the input points and how it affects the training process. Next, we show how to reduce the number of effective neurons in the latent layer.

%% file: sections/dimension.tex
\section{Reducing Input Dimension}{\label{sec:input}}
Let $\*X \in \~R^{n \times d}$ be a set of $n$ input points in $\~R^{d}$. Recall that we assume each point is normalized to a unit vector. We apply a Johnson-Lindenstrauss \cite{johnson1984extensions} transformation $\*C$ on the input $\*X$ before feeding it into the network. Every index of $\*C$ is an i.i.d. sample of $\@{N}(0,1/\sqrt{\ell})$, where $\ell = O(\log(n/\delta))$ for some $\delta \in (0,1)$. This transformation is drawn once before training, and not updated in the training process. 
We use a low rank weight matrix $\*W$ which is defined as $\*W = \*B\*C$, where $\*B \in \~R^{m \times \ell}$ and $\*C$ is the $\ell \times d$ JL transformation matrix. Now, based on the above low rank weight matrix $\*W$ for every input $\*X_{i}$ and network outputs $\*u_{i}$ as follows,
\begin{equation*}
  \*u_{i} = \frac{1}{\sqrt{m}}\sum_{r \leq m}\*v_{r}\cdot\sigma(\*B_{r}\*C\*X_{i}^{T}).
\end{equation*}

Notice, how $\*B$ and $\*C$ replace $\*W$ in equation \eqref{eq:W} in a plug and play fashion. Now we discuss the effect of this on the training phase. 
%
%
%
%
%
Fixing, $\*C$ (the dimension reduction layer) and $\*v$ the goal is to learn $\*B$ that minimizes,
\begin{equation}
    \min_{\*B} \frac{1}{2}\sum_{i \leq n} \left(\*y_{i} - \*u_{i}\right)^{2}. \label{eq:loss}
\end{equation}
The above cost is a function of $\*B$ and it can be represented in vector form as $\Phi(\*B) = \frac{1}{2}\|\*y - \*u\|^{2}$.
The matrix $\*C$ is first initialized as mentioned above.
The parameter $\*v$ is also drawn as a random vector sampled from $\{-1,+1\}^{m}$. Both $\*C$ and $\*v$ are fixed and do not vary during training.

We now randomly initialize the trainable parameters $\*B(0)$ at time $0$, where every entry is an i.i.d. sample from $\@{N}(0,1)$. Now with a fixed $\*C$ and $\*v$, we learn $\*B$ using gradient descent on the above loss function equation \eqref{eq:loss} based on a small learning rate $\eta$. Every row of $\*B_{r}$, $\forall r \in [m]$ are updated as follows,

\begin{eqnarray*}
  \*B_{r}(t+1) &=& \*B_{r}(t) - \eta\frac{\partial \Phi(\*B(t))}{\partial \*B_{r}(t)} \\
  &=& \*B_{r}(t) - \frac{\eta}{\sqrt{m}}\sum_{j \leq n}(\*y_{j} - \*u_{j}(t))\*v_{r}\*X_{j}\*C^{T}\mathbbm{1}\{\*B_{r}(t)\*C\*X_{j}^{T} \geq 0\}
\end{eqnarray*}

The NTK matrix $\*H^{\infty} \in \~R^{n \times n}$ is now defined as follows:
\begin{equation}
  \*H_{i,j}^{\infty} = \~E_{\*b}\left[\*X_{i}\*C^{T}\*C\*X_{j}^{T}\mathbbm{1}\{\*b\*C\*X_{i}^{T} \geq 0, \*b\*C\*X_{j}^{T} \geq 0\}\right]
\end{equation}

Here $\*b$ is a random vector from $\@{N}(\*0,\*I)$. The above term is the expected dot product between two vectors which corresponds to the change in the output for the inputs $\*X_{i}$ and $\*X_{j}$ with respect to $\*b$, i.e., $\langle\frac{\partial \*u_{i}}{\partial \*b},\frac{\partial \*u_{j}}{\partial \*b} \rangle$ (for details on this derivation, refer to \cite{arora2019fine}).

Now for a random initialization $\*B(0)$, fixed $\*C$ and $\*v$ we define $\*H(0) \in \~R^{n \times n}$ such that $\forall i,j \in [n]$,

\begin{equation}{\label{eq:kernel}}
  \*H_{i,j}(0) = \frac{1}{m}\sum_{r \leq m}\left[\*X_{i}\*C^{T}\*C\*X_{j}^{T}\mathbbm{1}\{\*B_{r}(0)\*C\*X_{i}^{T} \geq 0, \*B_{r}(0)\*C\*X_{j}^{T} \geq 0\}\right]
\end{equation}

\begin{lemma}{\label{lemma:init-fixed}}
 For $\*H^{\infty}$ and $\*H(0)$ as defined above, with at least $1-\delta$ probability,
 \begin{equation*}
   \|\*H^{\infty} - \*H(0)\| \leq O\left(\frac{n^{2}\log(n/\delta)}{m}\right)
 \end{equation*}
\end{lemma}

\begin{proof}{\label{proof:init-fixed}}
 Every single term in the difference matrix $\*H_{i,j}^{\infty} - \*H_{i,j}(0)$ is
 \begin{equation*}
   \*X_{i}\*C^{T}\*C\*X_{i}\left(\~E[\mathbbm{1}\{\*b\*C\*X_{i}^{T} \geq 0, \*b\*C\*X_{j}^{T} \geq 0\}] - \frac{1}{m}\sum_{r \leq m}\mathbbm{1}\{\*B_{r}(0)\*C\*X_{i}^{T} \geq 0, \*B_{r}(0)\*C\*X_{j}^{T} \geq 0\}\right)
 \end{equation*}
 Here the randomness is over $\*b \in \~R^{\ell}$, which takes $m$ i.i.d random samples from $\@{N}(\*0,\*I_{\ell})$. For all $r \in [m]$ the random variable $\mathbbm{1}\{\*B_{r}(0)\*C\*X_{i}^{T} \geq 0, \*B_{r}(0)\*C\*X_{j}^{T} \geq 0\} \in [0,1]$. Now since the above term is the difference between the expectation of a random variable and the empirical average of $m$ i.i.d samples of the random variable, so by Hoeffding's inequality we get the following with at least $1-\delta'$ probability,
 \begin{equation*}
   |\*H_{i,j}^{\infty} - \*H_{i,j}(0)| \leq |\*X_{i}\*C^{T}\*C\*X_{j}^{T}|\frac{\log(2/\delta')}{2m}.
 \end{equation*}
 Next, for a random $\*C \in \~R^{\ell \times d}$ such that each entry of $\*C$ is an i.i.d. sample from $\@{N}(0,1/\sqrt{\ell})$ where $\ell = O(\frac{\log (n/\delta)}{\varepsilon^{2}})$, for all $i,j \in [n]$ we have $\~P(|\*X_{i}\*C^{T}\*C\*X_{j}^{T} - \*X_{i}\*X_{j}^{T}| \geq \varepsilon) < \delta/(2n^{2})$. Now for $\varepsilon = O(1)$, setting $\delta' = \delta/(2n^{2})$ and taking a union bound over all pairs of $(i,j) \in [n]$ we get the following with at least $1-\delta$,
 \begin{eqnarray*}
  \|\*H^{\infty} - \*H(0)\| \leq \sum_{i,j}|\*H_{i,j}^{\infty} - \*H_{i,j}(0)| \leq O\left(n^{2}\frac{\log(2n^{2}/\delta)}{2m}\right)
 \end{eqnarray*}
\end{proof}

Now with the following lemma, we show that $\*H(0)$ does not change too much during the training period. 

\begin{lemma}{\label{lemma:boundR}}
 If $\*B_{1}, \ldots, \*B_{m}$ are random vectors whose entries are iid sample from $\@{N}(0,1)$ then with probability $1-\delta$, the following holds. 
 
 For any set of vectors $\*B_{1}(0), \ldots, \*B_{m}(0)$ if $\|\*B_{r}\*C - \*B_{r}(0)\*C\| \leq O\left(\frac{\delta \lambda_{0}}{n^{2}}\right) \stackrel{\Delta}{=} R$, then with $\*B$ the matrix $\*H \in \~R^{n \times n}$ whose $(i,j)^{th}$ entries are defined as \eqref{eq:kernel} 
 satisfies $\|\*H - \*H(0)\| \leq \frac{\lambda_{0}}{4}$ and $\lambda_{\min}(\*H) \geq \frac{\lambda_{0}}{4}$.
\end{lemma}

\begin{proof}{\label{proof:boundR}}
 For a fixed $\*C$ and a random $\*B_{r}$ we define a random vector $\*W_{r} = \*B_{r}\*C$. Similarly we also define $\*W_{r}(0) = \*B_{r}(0)\*C$. Now consider the following event for every input $\*X_{i} \in \*X$ and neuron $r \in [m]$,
 \begin{equation*}
  \@{E}_{r,i} = \{\exists \*B_{r} : \|\*W_{r} - \*W_{r}(0)\| \leq R, \mathbbm{1}\{\*W_{r}(0)\*X_{i}^{T} \geq 0\} \neq \mathbbm{1}\{\*W_{r}\*X_{i}^{T} \geq 0\}\}.
 \end{equation*}
 The above event happens if and only if $|\*W_{r}(0)\*X_{i}^{T}| \leq R$. Now notice that for a fixed $\*C$ and a randomly initialized $\*B_{r}(0)$ we have $z \sim \@{N}(0,1)$ where $z = \*B_{r}(0)\*C\*X_{i}^{T}$. So $\~P(\@{E}_{r,i}) = \~P(|z| \leq R) \leq \int_{-R}^{R}\frac{1}{\sqrt{2\pi}}\exp(-x^{2})dx \leq \frac{\sqrt{2}R}{\sqrt{\pi}}$. Now we bound deviation of every entry $(i,j) \in [n] \times [n]$ as follows,
 \begin{eqnarray*}
  \~E[|\*H_{i,j} - \*H_{i,j}(0)|] &=& \~E\bigg[O\left(\frac{1}{m}\right)\*X_{i}\*C^{T}\*C\*X_{j}^{T}\sum_{r \leq m}(\mathbbm{1}\{\*B_{r}\*C\*X_{i}^{T} \geq 0, \*B_{r}\*C\*X_{j}^{T} \geq 0\} \\
  &-& \mathbbm{1}\{\*B_{r}(0)\*C\*X_{i}^{T} \geq 0, \*B_{r}(0)\*C\*X_{j}^{T} \geq 0\})\bigg] \\
  &\leq& O\left(\frac{1}{m}\right)\sum_{r \leq m}\~E[\mathbbm{1}\{\@{E}_{r,i}\}\cup \mathbbm{1}\{\@{E}_{r,j}\}] 
  \leq O\left(\frac{2\sqrt{2}R}{\sqrt{\pi}}\right) \leq O(R)
 \end{eqnarray*}
 So, $\~E[\sum_{i,j}|\*H_{i,j} - \*H_{i,j}(0)|] \leq O(n^{2}R)$ and by Markov inequality we have $\sum_{i,j}|\*H_{i,j} - \*H_{i,j}(0)| \leq O\left(\frac{n^{2}R}{\delta}\right)$ with at least $1-\delta$ probability. Therefore with at least $1-\delta$ probability $\|\*H - \*H(0)\| \leq O\left(\frac{n^{2}R}{\delta}\right)$. Finally,
 \begin{equation*}
  \lambda_{\min}(\*H) \geq \lambda_{\min}(\*H(0)) - O\left(\frac{n^{2}R}{\delta}\right) \geq \frac{\lambda_{0}}{2}
 \end{equation*}
\end{proof}


Now we state our main theorem of this section describing the required network width and the learning rate to get to the desired training loss. 

\begin{theorem}{\label{thm:trainingDu}}
 Assume $\lambda_0 = \lambda_{\min}(\*H^{\infty}) > 0$. Fix $\delta \in (0,1)$, and take the number of neurons in the latent layer $m$ to be $\Omega\left(\frac{n^{6}}{\lambda_{0}^{4}\delta^{2}}\right)$.
 $\*B(0)$ is a random matrix whose entries are i.i.d. sample from $\@{N}(0,1)$ and we let the learning rate $\eta$ be  $O\left(\frac{\lambda_{0}}{n^{2}}\right)$ then during the gradient descent we have the following with probability at least $1-\delta$ for $t = 0, 1, \ldots$,
 \begin{equation}
   \|\*u(t) - \*y\|^{2} \leq \left(1-\frac{\eta\lambda_{0}}{2}\right)^{t}\|\*u(0) - \*y\|^{2}
 \end{equation} 
\end{theorem}

\begin{proof}{\label{proof:trainingDu}}
 We prove it by induction. The induction hypothesis is the following,
 \begin{equation}
  \|\*u(t) - \*y\|^{2} \leq \left(1-\frac{\eta\lambda_{0}}{2}\right)^{t}\|\*u(0) - \*y\|^{2}
 \end{equation}

The following corollary follows from the hypothesis.
\begin{corollary}{\label{cor:neronSimple}}
 If the above hypothesis is true then we have the following for all $r \in [m]$.
  \begin{equation}
  \|\*W_{r}(t+1) - \*W_{r}(0)\| \leq \frac{4\sqrt{n}\|\*u(0) - \*y\|}{\sqrt{m}\lambda_{0}} \stackrel{\Delta}{=} R'
 \end{equation}
\end{corollary}

\begin{proof}

 \begin{eqnarray}
  \|\*W_{r}(t+1) - \*W_{r}(0)\| &\leq& \sum_{t' \leq t}\left\|-\eta\frac{\partial \Phi(\*B(t'))}{\partial \*B_{r}(t')}\*C\right\| \nonumber \\
  &=& \eta\sum_{t' \leq t}\left\|\sum_{j \leq n}\frac{(\*y_{j} - \*u_{j}(t'))\*v_{r}\*X_{j}\*C^{T}\*C(\mathbbm{1}\{\*B_{r}(t')\*C\*X_{j}^{T} \geq 0\})}{\sqrt{m}} \right\| \nonumber \\
  &\leq& \eta\sum_{t' \leq t}\left\|\sum_{j \leq n}\frac{(\*y_{j} - \*u_{j}(t'))\*X_{j}^{T}\*C^{T}\*C}{\sqrt{m}} \right\| \label{eq:cor4eq1} \\
  &\leq& \frac{\eta}{\sqrt{m}}\sum_{t' \leq t}\left(\sum_{j \leq n}|\*y_{j} - \*u_{j}(t')|\right)\max_{j}\|\*X_{j}\*C^{T}\*C\| \label{eq:cor4eq2} \\
  &\leq& O\left(\frac{\eta}{\sqrt{m}}\right)\sum_{t' \leq t}\left(\sum_{j \leq n}|\*y_{j} - \*u_{j}(t')|\right) \label{eq:cor4eq3} \\
  &\leq& O\left(\frac{\eta}{\sqrt{m}}\right)\sum_{t' \leq t}\sqrt{n}\|\*y - \*u(t')\| \label{eq:cor4eq4} \\
  &\leq& O\left(\frac{\eta\sqrt{n}}{\sqrt{m}}\right)\sum_{t' \leq t}\left(1-\frac{\eta\lambda_{0}}{2}\right)^{t'/2}\|\*u(0) - \*y\| \label{eq:cor4eq5} \\
  &\leq& O\left(\frac{\eta\sqrt{n}}{\sqrt{m}}\right)\sum_{t'=0}^{\infty}\left(1-\frac{\eta\lambda_{0}}{2}\right)^{t'/2}\|\*u(0) - \*y\|
  = O\left(\frac{4\sqrt{n}}{\sqrt{m}\lambda_{0}}\right)\|\*u(0) - \*y\| \nonumber
 \end{eqnarray}

In \eqref{eq:cor4eq1} we upper bound $\mathbbm{1}\{\*b_{r}(t')\*C\*x_{j}^{T} \geq 0\}$ by $1$. In \eqref{eq:cor4eq2} we used holder's inequality. In \eqref{eq:cor4eq3} we use the upper bound $|\*x_{i}\*C^{T}\*C| \leq O(1)$. Using Cauchy–Schwartz we get \eqref{eq:cor4eq4}. In \eqref{eq:cor4eq5} we used our hypothesis. 
\end{proof}

Notice that the hypothesis is trivially true for the base case $t = 0$. Suppose it holds true for $t' = 0, 1, \ldots, t$ and now we show that it also holds for $t' = t+1$. For this we assume the same event $\@{E}_{r,i}$ for all input $\*x_{i}$ and neuron $r$. We know that, $R = O\left(\frac{\delta\lambda_{0}}{n^{2}}\right)$. Let $\*S_{i} = \{r \in [m]: \mathbbm{1}\{\@{E}_{r,i}\} = 0\}$ and $\*S_{i}^{\perp} = [m] \backslash \*S_{i}$. Note that $\@{E}_{r,i}$ is non empty if and only if $|\*W_{r}(0)\*X_{i}^{T}| \leq R$. Now for fixed $\*C$ as defined above and a random $\*B_{r}(0)$, $z \sim \@{N}(0,1)$ where $z = \*B_{r}(0)\*C\*X_{i}^{T} $. So we have, 
\begin{equation*}
  \~P(\@{E}_{r,i}) = \~P(|z| \leq R) \leq \int_{-R}^{R}\frac{1}{\sqrt{2\pi}}exp(-x^{2})dx \leq \frac{2R}{\sqrt{2\pi}} \leq R.
\end{equation*}

So,
\begin{equation*}
  \~E[|\*S_{i}^{\perp}|] = \sum_{r \leq m}\~P(\@{E}_{r,i}) \leq mR.
\end{equation*}

Further, $\~E[\sum_{i \leq n}|\*S_{i}^{\perp}|] \leq mnR$. So by using Markov inequality we have the following with at least $1-\delta$ probability.
\begin{equation*}
  \sum_{i \leq n}|\*S_{i}^{\perp}| \leq \frac{mnR}{\delta}
\end{equation*}
Here $C$ is some appropriate constant. Now consider the following term,
\begin{eqnarray}
 \*u_{i}(t+1) - \*u_{i}(t) &=& \frac{1}{\sqrt{m}}\sum_{r \leq m}v_{r}\left[\sigma(\*B_{r}(t+1)\*C\*X_{i}^{T}) - \sigma(\*W_{r}(t)\*C\*X_{i}^{T})\right] \nonumber \\
 &=& \frac{1}{\sqrt{m}}\sum_{r \leq m}v_{r}\left[\sigma\left(\left(\*B_{r}(t) - \eta\frac{\partial \Phi(\*B(t))}{\partial \*B_{r}(t)}\right)\*C\*X_{i}^{T}\right) - \sigma(\*B_{r}(t)\*C\*X_{i}^{T})\right] \nonumber \\
 &=& I^{i}_{1} + I^{i}_{2} \nonumber
\end{eqnarray}

Here,
\begin{eqnarray}
 I^{i}_{1} &=& \frac{1}{\sqrt{m}}\sum_{r \in \*S_{i}}\*v_{r}\left[\sigma\left(\left(\*B_{r}(t) - \eta\frac{\partial \Phi(\*B(t))}{\partial \*B_{r}(t)}\right)\*C\*X_{i}^{T}\right) - \sigma(\*B_{r}(t)\*C\*X_{i}^{T})\right] \nonumber \\
 I^{i}_{2} &=& \frac{1}{\sqrt{m}}\sum_{r \in \*S_{i}^{\perp}}\*v_{r}\left[\sigma\left(\left(\*B_{r}(t) - \eta\frac{\partial \Phi(\*B(t))}{\partial \*B_{r}(t)}\right)\*C\*X_{i}^{T}\right) - \sigma(\*B_{r}(t)\*C\*X_{i}^{T})\right] \nonumber
\end{eqnarray}

First we upper bound $I^{i}_{2}$.
\begin{eqnarray}
 |I^{i}_{2}| &=& \left|\frac{1}{\sqrt{m}}\sum_{r \in \*S_{i}^{\perp}}\*v_{r}\left[\sigma\left(\left(\*B_{r}(t) - \eta\frac{\partial \Phi(\*B(t))}{\partial \*B_{r}(t)}\right)\*C\*X_{i}^{T}\right) - \sigma(\*B_{r}(t)\*C\*X_{i}^{T})\right]\right| \nonumber \\
 &\leq& \left|\frac{1}{\sqrt{m}}\sum_{r \in \*S_{i}^{\perp}}\left[ - \eta\left(\frac{\partial \Phi(\*B(t))}{\partial \*B_{r}(t)}\right)\*C\*X_{i}^{T} \right]\right| \label{eq:thm4eq1}\\
 &\leq& \frac{\eta}{\sqrt{m}}\left|\sum_{r \in \*S_{i}^{\perp}}\sum_{j \leq n}\frac{(\*y_{j}-\*u_{j}(t))\*v_{r}\*X_{j}\*C^{T}\*C\*X_{i}^{T}\mathbbm{1}\{\*B_{r}(t)\*C\*X_{j} \geq 0\}}{\sqrt{m}}\right| \nonumber \\
 &\leq& \frac{\eta}{\sqrt{m}}\left|\sum_{r \in \*S_{i}^{\perp}}\max_{j}|\*X_{j}\*C^{T}\*C\*X_{i}^{T}|\sum_{j \leq n}|\frac{(\*y_{j}-\*u_{j}(t))\*v_{r}\mathbbm{1}\{\*B_{r}(t)\*C\*X_{j}^{T} \geq 0\}|}{\sqrt{m}}\right| \nonumber \\
 &\leq& \frac{2\eta|\*S_{i}^{\perp}|}{m}\sum_{j \leq n}|(\*y_{j}-\*u_{j}(t))\*v_{r}| \leq \frac{2\eta\sqrt{n}|\*S_{i}^{\perp}|}{m}\|\*y-\*u(t)\| \nonumber
\end{eqnarray}
Since ReLU is $1$-Lipschitz function and $|\*v_r| = 1$, so we get \eqref{eq:thm4eq1}. 

Now we analyze $I^{i}_{1}$. From Corollary \ref{cor:neronSimple} we have $\|\*W_{r}(t+1) - \*W_{r}(0)\| \leq R'$ and $\|\*W_{r}(t+1) - \*W_{r}(t)\| \leq R'$. Now to always ensure $R' < R$ we have,
\begin{eqnarray}
 \frac{4\sqrt{n}}{\sqrt{m}\lambda_{0}}\|\*y - \*u(0)\| &\leq& O\left(\frac{\delta\lambda_{0}}{n^{2}}\right) \\
 \frac{4\sqrt{n}}{\sqrt{m}}O(\sqrt{n}) &\leq& O\left(\frac{\delta\lambda_{0}^{2}}{n^{2}}\right) \\
 \frac{O(n)}{\sqrt{m}} &\leq& O\left(\frac{\delta\lambda_{0}^{2}}{n^{2}}\right) \\
 \frac{O(n^{3})}{\sqrt{m}} &\leq& O\left(\delta\lambda_{0}^{2}\right) \\
 m &\geq& O\left(\frac{n^{6}}{\lambda_{0}^{4}\delta^{2}}\right)
\end{eqnarray}
 
 As, $R' < R$, so for $I^{i}_{1}$ we have $\mathbbm{1}\{\*W_{r}(t+1)\*X_{i}^{T}\} = \mathbbm{1}\{\*W_{r}(t)^{T}\*x_{i}^{T}\}$. Now $I^{i}_{1} = \frac{du_{i}(t)}{dt}$, so we express $I^{i}_{1}$ as follows,
\begin{eqnarray}
 I^{i}_{1} &=& -\frac{\eta}{m}\sum_{j \leq n}\*X_{i}\*C^{T}\*C\*X_{j}^{T}(\*y_{j} - \*u_{j}(t))(\sum_{r \in \*S_{i}}\mathbbm{1}\{\*B_{r}(t)\*C\*X_{i}^{T} \geq 0, \*B_{r}(t)\*C\*X_{j}^{T} \geq 0\} \nonumber \\
 &=& -\eta\sum_{j \leq n}(\*u_{j}(t)-\*y_{j})(\*H_{i,j}(t) - \*H_{i,j}^{\perp}(t)) \nonumber
\end{eqnarray}

Here $\*H_{i,j}(t) = \sum_{r \leq m}\*X_{i}\*C^{T}\*C\*X_{j}^{T}\mathbbm{1}\{\*B_{r}(t)\*C\*X_{i}^{T} \geq 0, \*B_{r}(t)\*C\*X_{j}^{T} \geq 0\}$ and $\*H_{i,j}(t) = \sum_{r \in \*S_{i}^{\perp}}\*X_{i}\*C^{T}\*C\*X_{j}^{T}\mathbbm{1}\{\*B_{r}(t)\*C\*X_{i}^{T} \geq 0, \*B_{r}(t)\*C\*X_{j}^{T} \geq 0\}$. $\*H^{\perp}$ is an $n \times n$ psd matrix whose spectral norm can be bounded as follows,
\begin{equation*}
  \|\*H^{\perp}\| \leq \sum_{i,j}|\*H_{i,j}^{\perp}| \leq \frac{n}{m}\sum_{i \leq n} |\*X_{i}\*C^{T}\*C\*X_{j}^{T}| \cdot |\*S_{i}^{\perp}| \leq \frac{Cn^{2}R}{\delta}
\end{equation*}

Further,
\begin{eqnarray*}
  \|\*u(t+1) - \*u(t)\|^{2} &\leq& \eta^{2}\sum_{i \leq n} \left(\sum_{j \leq n}\frac{\partial \*u_{i}(t)}{dt} \right)^{2} \nonumber \\
  &\leq& \eta^{2}\sum_{i \leq n} \left(\sum_{j \leq n}\frac{1}{m}(\*y_{j} - \*u_{j}(t)) (\*X_{i}\*C^{T}\*C\*X_{j}^{T})\sum_{r \leq m}\mathbbm{1}\{\*B_{r}(t)\*C\*X_{i}^{T} \geq 0, \*B_{r}(t)\*C\*X_{j}^{T} \geq 0\} \right)^{2} \nonumber \\
  &\leq& \eta^{2}n \left(\max_{j} |\*X_{i}\*C^{T}\*C\*X_{j}^{T}| \sum_{j \leq n}|\*y_{j} - \*u_{j}(t)| \right)^{2} \nonumber \\
  &\leq& \eta^{2}n \left(\sum_{j \leq n}|\*y_{j} - \*u_{j}(t)| \right)^{2}
  \leq \eta^{2}n^{2} \|\*y - \*u(t)\|^{2} \nonumber
\end{eqnarray*}

Finally we bound our desired term,
\begin{eqnarray}
 \|\*y - \*u(t+1)\|^{2} &=& \|\*y - \*u(t) - (\*u(t+1) - \*u(t))\|^{2} \nonumber \\
 &=& \|\*y-\*u(t)\|^{2} + \|\*u(t+1) - \*u(t)\|^{2} - 2(\*y-\*u(t))^{T}(\*u(t+1) - \*u(t)) \nonumber \\
 &=& \|\*y-\*u(t)\|^{2} + \|\*u(t+1) - \*u(t)\|^{2} - 2(\*y-\*u(t))^{T}\*I_{2} \nonumber \\
 &-& 2\eta(\*y-\*u(t))^{T}\*H(t)(\*y-\*u(t)) + 2\eta(\*y-\*u(t))^{T}\*H(t)^{\perp}(\*y-\*u(t)) \nonumber \\
 &\leq& \left(1 + \eta^{2}n^{2} - \frac{\sum_{i \leq n}2\eta\sqrt{n}|\*S_{i}^{\perp}|}{m} - \eta\lambda_{0} + \frac{2C\eta n^{2}R}{\delta}\right)\|\*y - \*u(t)\|^{2} \nonumber \\
 &\leq& \left(1 - \frac{\eta\lambda_{0}}{2}\right)\|\*y - \*u(t)\|^{2}
 \leq \left(1 - \frac{\eta\lambda_{0}}{2}\right)^{t}\|\*y - \*u(0)\|^{2} \nonumber 
\end{eqnarray}
\end{proof}


Due to our low rank network architecture we only need to update $O\left(\frac{n^{6}\log(n)}{\lambda_{0}^{4}\delta^{2}}\right)$ parameters. Now by using $\*C$ as Fast Johnson Lindenstrauss transformation \cite{ailon2006approximate}, the $\*C\*X_{i}^{T}$ takes $O(d\log(\ell))$. We compute the transformation $\*C\*X^{T}$ in $O(nd\log(\ell))$. So the time taken to update from $\*B(t)$ to $\*B(t+1)$ is $O(nd\log (\ell) + nm\ell)$, i.e., $O\left(\frac{n^{7}\log(n)}{\lambda_{0}^{4}\delta^{2}}\right)$. Since, the guarantees are same upto a constant factor, so to ensure that $\|\*u(t) - \*y\|^{2} \leq \varepsilon \|\*u(0) - \*y\|^{2}$ even when $d = \Omega(n)$ the algorithm takes $\Omega\left(\frac{n^{7}\log(n)\log(\varepsilon)}{\lambda_{0}^{4}\delta^{2}\log(1-\eta\lambda_{0})}\right)$.



%% file: sections/neurons.tex
\newcommand{\nir}[1]{{\color{red} #1}}
\section{Reducing Effective Network Width}{\label{sec:neuron}}
The matrix $\*W$ represents the first layer parameters. 
We represent this weight matrix $\*W$ as the product
$\*A\*B\*C$. Here $\*A \in \~R^{m \times k}$, and as before $\*B \in \~R^{k \times \ell}$ and $\*C \in \~R^{\ell \times d}$.
In other words, we add a dimension reduction gadget as input to the latent layer. It is important to note that the number of neurons in the first layer remains the same. However, due to the low rank structure of the weight matrix $\*W$, not all neurons are independent of others. Hence effectively, in the training phase, not every neuron needs to be learned. Now we study the effect of this addition on the NTK bound.

Now we discuss how this affects the training phase. First we randomly and independently initialize $\*A \in \~R^{m \times k}$ such that all its entries are iid sampled from $\@{N}(0,1/\sqrt{k})$ for some $k < m$. We also initialize $\*C$ as mentioned in the previous section. Next, recall that $\*v$ is initialized as a random vector sampled from $\{-1,+1\}^{m}$. Now, we learn $\*B$ using gradient descent on the loss function $\Phi(\*B) = \frac{1}{2}\|\*y - \*u\|^{2}$ based on a small learning rate $\eta$ by keeping rest of the network parameters fixed. The matrix $\*B$ is updated as follows,

\begin{eqnarray*}
  \*B(t+1) &=& \*B(t) - \eta\frac{\partial \Phi(\*B(t))}{\partial \*B(t)} \\
  &=& \*B(t) - \frac{\eta}{\sqrt{m}}\sum_{j \leq n}(\*y_{j} - \*u_{j}(t))\*A^{T}\*Z_{j}(t)\*X_{j}\*C^{T}\ .
\end{eqnarray*}

Here $\*Z_{j}(t) \in \{-1,0,1\}^{m}$ is such that its $r^{th}$ index is $\*v_{r}\cdot\mathbbm{1}\{\*A_{r}\*B(t)\*C\*X_{j}^{T} > 0\}$ for all $r \in [m]$. Now given the input $\*X$, the gram matrix $\*H^{\infty} \in \~R^{n \times n}$ is defined with its $(i,j)^{th}$ term  as follows:
\begin{equation}
  \*H_{i,j}^{\infty} = \~E_{\*B}\left[\frac{1}{m}(\*X_{i}\*C^{T}\*C\*X_{j}^{T})(\*Z_{i}\*A\*A^{T}\*Z_{j}^{T})\right]\ .
\end{equation}

Note that $\*Z_{i}$ and $\*Z_{j}$ are random because they depend on $B$. Now for a random $\*B(0)$ upon initialization, we define $\*H(0) \in \~R^{n \times n}$  by $\forall i,j \in [n]$:

\begin{equation}{\label{eq:ker2}}
  \*H_{i,j}(0) = \frac{1}{m}(\*X_{i}\*C^{T}\*C\*X_{j}^{T}) (\*Z_{i}(0)\*A\*A^{T}\*Z_{j}^{T}(0))\ .
\end{equation}


\begin{lemma}{\label{lemma:ConditionalEvent}}
 If $\*B \in \~R^{k \times \ell}$ is a random matrix whose entries are iid sample from $\@{N}(0,1)$ then with probability $1-\delta$ the following holds.
 
 For any $\*B(0)$, if $\|\*A_{r}\*B\*C - \*A_{r}\*B(0)\*C\| \leq O\left(\frac{\delta^{2}\lambda_{0}}{n^{2}}\right) \stackrel{\Delta}{=} R$ then with $\*B$ the matrix $\*H \in \~R^{n \times n}$ whose $(i,j)^{th}$ entries are defined \eqref{eq:ker2},
 satisfies $\|\*H - \*H(0)\| \leq \frac{\lambda_{0}}{4}$ and $\lambda_{\min}(\*H) \geq \frac{\lambda_{0}}{4}$.
\end{lemma}
\begin{proof}
 For a fixed $\*A, \*C$ and a random $\*B$ we get a random vector as $\*W_{r} = \*A_{r}\*B\*C$. Similarly we also define $\*W_{r}(0)$ corresponding to $\*B(0)$. Now consider the following event for every input $\*X_{i} \in \*X$ and neuron $r \in [m]$,
 \begin{equation*}
   \@{E}_{r,i} = \{\exists \*B : \|\*W_{r} - \*W_{r}(0)\| \leq R, \mathbbm{1}\{\*W_{r}(0)\*X_{i}^{T} \geq 0\} \neq \mathbbm{1}\{\*W_{r}\*X_{i}^{T} \geq 0\}\}.
 \end{equation*}
 The above event happens if and only if $|\*W_{r}(0)\*X_{i}^{T}| \leq R$. Now notice that for a fixed $\*A_{r},\*C$ and a randomly initialized $\*B(0)$ we have $h \sim \@{N}(0,1)$ where $h = \*A_{r}\*B(0)\*C\*X_{i}^{T}$. So $\~P(\@{E}_{r,i}) = \~P(|h| \leq R) \leq \int_{-R}^{R}\frac{1}{\sqrt{2\pi}}\exp(-x^{2})dx \leq R$. Now we bound deviation of every entry $(i,j) \in [n] \times [n]$ as follows,
 \begin{eqnarray*}
  \~E[|\*H_{i,j} - \*H_{i,j}(0)|] &=& \~E\left[\frac{1}{m}\left|(\*X_{i}\*C^{T}\*C\*X_{j}^{T})(\*Z_{i}\*A\*A^{T}\*Z_{j}^{T} - \*Z_{i}(0)\*A\*A^{T}\*Z_{j}^{T}(0))\right|\right] 
 \end{eqnarray*}
 Now notice that,
 \begin{eqnarray*}
  (\*Z_{i}\*A\*A^{T}\*Z_{j}^{T} - \*Z_{i}(0)\*A\*A^{T}\*Z_{j}^{T}(0)) &=& \sum_{p,q}\*Z_{i,p}\*M_{p,q}\*Z_{j,q} - \*Z_{i,p}(0)\*M_{p,q}\*Z_{j,q}(0) \\
  &=&\sum_{p,q}\*M_{p,q}\*v_{p}\*v_{q}(\mathbbm{1}(\*A_{p}\*B\*C\*X_{i}^{T} \geq 0)\cdot \mathbbm{1}(\*A_{q}\*B\*C\*X_{j}^{T} \geq 0) \\
  &-& \mathbbm{1}(\*A_{p}\*B(0)\*C\*X_{i}^{T} \geq 0)\cdot \mathbbm{1}(\*A_{q}\*B(0)\*C\*X_{j}^{T} \geq 0)) 
 \end{eqnarray*}
 
 Here $\*M_{p,q}$ is the inner product between the $p^{th}$ and the $q^{th}$ row of $\*A$, i.e., $\*M_{p,q} = \*A_{p}\*A_{q}^{T}$. Next, since $\~E[|\mathbbm{1}(\*A_{p}\*B\*C\*X_{i} \geq 0)\cdot \mathbbm{1}(\*A_{q}\*B\*C\*X_{j} \geq 0) - \mathbbm{1}(\*A_{p}\*B(0)\*C\*X_{i} \geq 0)\cdot \mathbbm{1}(\*A_{q}\*B(0)\*C\*X_{j} \geq 0)|] \leq \~P(\@{E}_{p,i}) + \~P(\@{E}_{q,j})$, hence $\mathbbm{1}(\*A_{p}\*B\*C\*X_{i} \geq 0)\cdot \mathbbm{1}(\*A_{q}\*B\*C\*X_{j} \geq 0) - \mathbbm{1}(\*A_{p}\*B(0)\*C\*X_{i} \geq 0)\cdot \mathbbm{1}(\*A_{q}\*B(0)\*C\*X_{j} \geq 0)$ is either $-1$ or $1$ with probability at most $2R$. Here the randomness is due to the random matrix $\*B$. So we can rewrite, $(\*Z_{i}\*A\*A^{T}\*Z_{j}^{T} - \*Z_{i}(0)\*A\*A^{T}\*Z_{j}^{T}(0)) = \*x\*A\*A^{T}\*y$ where $\*x$ and $\*y$ are $m$ dimensional vectors such that with at least $1-\delta/2$ probability $\|\*x\|^{2} = \|\*y\|^{2} \leq 
 \frac{2mR}{\delta}$. 
 We use the following claim, to discuss the desired property of $\*A$.
 \begin{claim}{\label{claim1}}
   Let $\@{S}$ be a set such that every element $\*s \in \@{S}$ is an $m$ dimensional vector in $\{-1,0,+1\}^{m}$ and $\|\*s\|_{1} \leq \frac{2mR}{\delta}$. The size of such set $|\@{S}| = O(\binom{m}{mR/\delta}\cdot 3^{mR/\delta})$. Let $k = O(\log(|\@{S}|)) = \tilde{O}\left(\frac{mR}{\delta}\right)$ and recall matrix $\*A$ has random i.i.d. samples from $\@{N}(0,1/\sqrt{k})$. Now for all $\*s \in \@{S}$ we get, $\|\*s\*A\| = O(\|\*s\|) = O(\sqrt{mR/\delta})$ with high probability.
 \end{claim}
  Now the above claim yields, $\~E[|\*H_{i,j} - \*H_{i,j}(0)|] \leq \frac{1}{m}(O(1)\sqrt{mR/\delta}\sqrt{mR/\delta}) = O(R/\delta)$. So, $\~E[\sum_{i,j}|\*H_{i,j} - \*H_{i,j}(0)|] \leq O\left(\frac{n^{2}R}{\delta}\right)$ and by Markov inequality we have $\sum_{i,j}|\*H_{i,j} - \*H_{i,j}(0)| \leq O\left(\frac{n^{2}R}{\delta^{2}}\right)$ with at least $1-\delta$ probability. Therefore with at least $1-\delta$ probability $\|\*H - \*H(0)\| \leq O\left(\frac{n^{2}R}{\delta^{2}}\right)$. Finally,
 \begin{equation*}
   \lambda_{\min}(\*H) \geq \lambda_{\min}(\*H(0)) - O\left(\frac{n^{2}R}{\delta^{2}}\right) \geq \frac{\lambda_{0}}{2}
 \end{equation*}
\end{proof}

Note that the above lemma ensures that $R = O\left(\frac{\lambda_{0}\delta^{2}}{n^{2}}\right)$. Further, it implies that from the point of initialization of $\*B(0)$, out of $m$ neurons only $O\left(\frac{mR}{\delta}\right)$ of them change their activation function value from their initial function value. The next lemma states the required learning rate $\eta$ and the number of neurons $m$ to get the desired training error after certain iterations. 

\begin{theorem}{\label{thm:train}}
 Let $\lambda_0$ is the smallest non zero eigenvalue of $\*H^{\infty}$. For $\delta \in (0,1)$, number of neurons $m = \Omega\left(\frac{n^{6}}{\lambda_{0}^{4}\delta^{7}}\right)$, $\*B(0)$ is randomly initialized and we set $\eta = O\left(\frac{\lambda_{0}\delta^{4}}{n^{2}}\right)$. Then during  gradient descent we have the following with probability at least $1-\delta$ for $t = 0, 1, \ldots$,
 \begin{equation}
   \|\*u(t) - \*y\|^{2} \leq \left(1-\frac{\eta\lambda_{0}}{2}\right)^{t}\|\*u(0) - \*y\|^{2}
 \end{equation} 
\end{theorem}

\begin{proof}
 We prove this by induction. The induction hypothesis is the following,
 \begin{equation}
   \|\*u(t) - \*y\|^{2} \leq \left(1-\frac{\eta\lambda_{0}}{2}\right)^{t}\|\*u(0) - \*y\|^{2}
 \end{equation}

Notice that the hypothesis is trivially true for the base case $t = 0$. Suppose it holds true for $t' = 0, 1, \ldots, t$ and now we show that it also holds for $t' = t+1$. We have,
\begin{eqnarray*}
  \|\*u(t+1) - \*u(t)\|^{2} &\leq& \eta^{2}\sum_{i \leq n} \left(\frac{\partial \*u_{i}(t)}{dt} \right)^{2} \nonumber \\
  &\leq& \eta^{2}\sum_{i \leq n} \left(\sum_{j \leq n}\frac{1}{m}(\*y_{j} - \*u_{j}(t))(\*X_{i}\*C^{T}\*C\*X_{j}^{T})(\*Z_{i}(t)\*A\*A^{T}\*Z_{j}^{T}(t))\right)^{2} \nonumber \\
  &\leq& \eta^{2}n \left(\frac{1}{m}\max_{j} |\*X_{i}\*C^{T}\*C\*X_{j}^{T}|\sum_{j \leq n}\left|(\*Z_{i}(t)\*A^{T}\*A\*Z_{j}^{T}(t)) \left(\*y_{j} - \*u_{j}(t)\right)\right| \right)^{2} \nonumber \\
  &\leq& \eta^{2}n \left(\frac{O(1)}{m}\left(\sum_{j \leq n}(\*Z_{i}(t)\*A\*A^{T}\*Z_{j}^{T}(t))^{2}\right)^{1/2} \|\*y - \*u(t)\| \right)^{2} \nonumber \\
  &\leq& \eta^{2}n \left(\frac{O(1)}{m}\left(\sum_{j \leq n}(\|\*Z_{i}(t)\*A\|\|\*A^{T}\*Z_{j}^{T}(t)\|)^{2}\right)^{1/2} \|\*y - \*u(t)\| \right)^{2} \nonumber
\end{eqnarray*}

Notice that,
\begin{eqnarray*}
 \|\*Z_{i}(t)\*A\|^{2} = \sum_{r_{1} \leq k} (\*Z_{i}(t)\*A_{r_{1}}^{T})^{2} = \sum_{r_{1} \leq k}\left(\sum_{r_{2} \leq m}\*v_{r_{2}}\mathbbm{1}(\*A_{r_{2}}\*B(t)\*C\*X_{i}^{T} \geq 0)\*A_{r_{2},r_{1}}\right)^{2} 
 = \sum_{r_{1} \leq k}(\*v\tilde{\*A}_{r_{1}}^{T})^{2} = \|\*v\tilde{\*A}\|^{2}
\end{eqnarray*}

Here $\tilde{\*A}_{r_{1}}$ represents the entrywise vector product between $\*A_{r_1}$ and corresponding indicator variable. Notice that $\|\tilde{\*A}\|_{F}^{2} \leq \|\*A\|_{F}^{2}$ and with at least $1-\delta/2$ probability $\|\*A\|_{F}^{2} \leq 2m/\delta$. So, for all $i \in [n]$, $\~E_{\*v}[\|\*Z_{i}(t)\*A\|^{2}] = \sum_{r_{1}}\|\tilde{\*A}_{r_{1}}\|^{2} \leq \|\*A\|_{F}^{2} \leq 2m/\delta$. Now by Markov we get, $\sum_{i}\|\*Z_{i}(t)\*A\|^{2} \leq O(nm/\delta^{2})$ with at least $1-\delta$ probability. Now,

\begin{eqnarray*}
  \|\*u(t+1) - \*u(t)\|^{2} &\leq& \eta^{2}n \left(\frac{O(1)}{m}\left(\sum_{j \leq n}(\|\*Z_{i}(t)\*A^{T}\|\|\*A\*Z_{j}^{T}(t)\|)^{2}\right)^{1/2} \|\*y - \*u(t)\| \right)^{2} \nonumber \\
  &\leq& \eta^{2}n \left(\frac{O(1)}{m}\left(\frac{nm^{2}}{\delta^{4}}\right)^{1/2} \|\*y - \*u(t)\| \right)^{2} 
  \leq O\left(\frac{\eta^{2}n^{2}}{\delta^{4}}\right) \|\*y - \*u(t)\|^{2} \nonumber
\end{eqnarray*}

From the above analysis notice that,
\begin{eqnarray}
 |\*u_{i}(t+1) - \*u_{i}(0)| &\leq& \sum_{t' < t}\left|\frac{\partial \*u_{i}(t')}{dt'} \right| \leq O\left(\frac{\eta}{\delta^{2}}\right) \sum_{t' < t}\sqrt{n}\|\*y - \*u(t)\| \label{eq} \\
  &\leq& O\left(\frac{\eta\sqrt{n}}{\delta^{2}}\right) \sum_{t' < t}\left(1-\frac{\eta\lambda_{0}}{2}\right)^{t'/2}\|\*u(0) - \*y\| \nonumber \\
  &\leq& O\left(\frac{\eta\sqrt{n}}{\delta^{2}}\right) \sum_{t' = 1}^{\infty}\left(1-\frac{\eta\lambda_{0}}{2}\right)^{t'/2}\|\*u(0) - \*y\| 
  \leq \frac{\sqrt{n}}{\delta^{2}\lambda_{0}}\|\*u(0) - \*y\| \leq \frac{n}{\delta^{5/2}\lambda_{0}} \nonumber
\end{eqnarray}

Our network parameters at time $t = 0$, ensures that the output $\*u_{i}(0)$ is $O(1)$, for all $i \in [n]$. Further we know that $\*y_{i} = O(1)$. Hence by Markov's inequality we have $\|\*u(0) - \*y\|^{2} \leq n/\delta$ with at least probability $1-\delta$.

 SO  
Now, based on the events $\@{E}_{r,i}$ for all input $\*X_{i}$ and neuron $r$, let $\*u_{i}$ be the network output for $\*X_{i}$. Next, for all $r \in [m]$, since the change from $\*W_{r}(0)$ to $\*W_{r}$ is bounded, so we have $|\*u_{i} - \*u_{i}(0)| = \frac{1}{\sqrt{m}}\*v(\sigma(\*W\*X_{i}^{T}) - \sigma(\*W(0)\*X_{i}^{T})) \leq \frac{1}{\sqrt{m}}\|\*v\|\|\*W - \*W(0)\|\|\*X_{i}\| \leq \frac{1}{\sqrt{m}} \sqrt{m} \sqrt{m}R = \sqrt{m}R$. Now at each step $t$ of the gradient descent the parameter $\*W_{r}(t)$ corresponding to every neuron $r \in [m]$ are affected due to the change in $\*B(t)$. So in order to insure that the condition in lemma \ref{lemma:ConditionalEvent} holds, we need ensure $|\*u_{i}(t+1) - \*u_{i}(0)| \leq |\*u_{i} - \*u_{i}(0)|$. So we get, $m = \Omega\left(\frac{n^{6}}{\lambda_{0}^{4}\delta^{7}}\right)$. Further, we know that, $R = O\left(\frac{\delta^{2}\lambda_{0}}{n^{2}}\right)$ and $\~P(\@{E}_{r,i}) \leq R$. Let $\*S_{i} = \{r \in [m]: \mathbbm{1}\{\@{E}_{r,i}\} = 0\}$ and $\*S_{i}^{\perp} = [m] \backslash \*S_{i}$. Let $\*s_{i} \in \~R^{m}$ be such that $\*s_{i,r} = 1, \forall r \in [\*S_{i}^{\perp}]$. So, $\~E[\|\*s_{i}\|^{2}] \leq mR$. Further, $\~E[\sum_{i \leq n}|\*S_{i}^{\perp}|] \leq mnR$. So by using Markov inequality we have $\sum_{i \leq n}|\*S_{i}^{\perp}| \leq \frac{mnR}{\delta}$ with at least $1-\delta$ probability. Now consider the following term,
\begin{eqnarray}
 \*u_{i}(t+1) - \*u_{i}(t) &=& \frac{1}{\sqrt{m}}\sum_{r \leq m}\*v_{r}\left[\sigma(\*A_{r}\*B(t+1)\*C\*X_{i}^{T}) - \sigma(\*A_{r}\*B(t)\*C\*X_{i}^{T})\right] \nonumber \\
 &=& \frac{1}{\sqrt{m}}\sum_{r \leq m}\*v_{r}\left[\sigma\left(\*A_{r}\left(\*B(t) - \eta\frac{\partial \Phi(\*B(t))}{\partial \*B(t)}\right)\*C\*X_{i}^{T}\right) - \sigma(\*A_{r}\*B(t)\*C\*X_{i}^{T})\right] \nonumber \\
 &=& I^{i}_{1} + I^{i}_{2} \nonumber
\end{eqnarray}

We analyze the term based on two sets of neurons. We have $I^{i}_{1}$ corresponding to the neurons which do not change their activation function value based for the network parameter $\*B(t)$ and $\*B(t+1)$. The term $I^{i}_{2}$ corresponds to rest of the neurons (that changes).
\begin{eqnarray}
 I^{i}_{1} &=& \frac{1}{\sqrt{m}}\sum_{r \in \*S_{i}}\*v_{r}\left[\sigma\left(\*A_{r}\left(\*B(t) - \eta\frac{\partial \Phi(\*B(t))}{\partial \*B(t)}\right)\*C\*X_{i}^{T}\right) - \sigma(\*A_{r}\*B(t)\*C\*X_{i}^{T})\right] \nonumber \\
 I^{i}_{2} &=& \frac{1}{\sqrt{m}}\sum_{r \in \*S_{i}^{\perp}}\*v_{r}\left[\sigma\left(\*A_{r}\left(\*B(t) - \eta\frac{\partial \Phi(\*B(t))}{\partial \*B(t)}\right)\*C\*X_{i}^{T}\right) - \sigma(\*A_{r}\*B(t)\*C\*X_{i}^{T})\right] \nonumber
\end{eqnarray}

We first upper bound $|I^{i}_{2}|$. 

\begin{eqnarray}
 |I^{i}_{2}| &=& \left|\frac{1}{\sqrt{m}}\sum_{r \in \*S_{i}^{\perp}}\*v_{r}\left[\sigma\left(\*A_{r}\left(\*B(t) - \eta\frac{\partial \Phi(\*B(t))}{\partial \*B(t)}\right)\*C\*X_{i}^{T}\right) - \sigma(\*A_{r}\*B(t)\*C\*X_{i}^{T})\right]\right| \nonumber \\
 &\leq& \left|\frac{1}{\sqrt{m}}\sum_{r \in \*S_{i}^{\perp}}\left[ - \eta\*A_{r}\left(\frac{\partial \Phi(\*B(t))}{\partial \*B(t)}\right)\*C\*X_{i}^{T} \right]\right| \label{eq:thm1eq1} \\
 &\leq& \frac{\eta}{\sqrt{m}}\|\*s_{i}\|\cdot\left\|\sum_{j \leq n}\frac{(\*y_{j}-\*u_{j}(t))(\*X_{j}\*C^{T}\*C\*X_{i}^{T})(\*A_{\*S_{i}^{\perp}}\*A^{T}\*Z_{j}^{T}(t))}{\sqrt{m}}\right\| \label{eq:thm1eq2} \\
 &\leq& \eta \sqrt{\frac{mR}{m\delta}}\left\|\sum_{j \leq n}\frac{(\*y_{j}-\*u_{j}(t))(\*X_{j}\*C^{T}\*C\*X_{i}^{T})(\*A_{\*S_{i}^{\perp}}\*A^{T}\*Z_{j}^{T}(t))}{\sqrt{m}}\right\| \nonumber \\
 &\leq& \eta\sqrt{\frac{R}{m\delta}}\left\|\max_{j}(\*X_{j}\*C^{T}\*C\*X_{i}^{T})(\*A_{\*S_{i}^{\perp}}\*A^{T}\*Z_{j}^{T}(t))\right\|\sum_{j \leq n}|\*y_{j}-\*u_{j}(t)| \label{eq:thm1eq3} \\
 &\leq& \eta\sqrt{\frac{R}{m\delta}}\left\|\max_{j}\*A_{\*S_{i}^{\perp}}\*A^{T}\*Z_{j}^{T}(t)\right\|\sqrt{n}\|\*y-\*u(t)\| \nonumber \\
 &\leq& \eta\sqrt{\frac{nR}{m\delta}}\|\*A_{\*S_{i}^{\perp}}\|\|\*A\|\max_{j}\|\*Z_{j}(t)\|\|\*y-\*u(t)\| \nonumber \\
 &\leq& \eta\sqrt{\frac{nR}{m\delta}}\sqrt{\frac{m}{k}}\sqrt{m}\|\*y-\*u(t)\| \label{eq:thm1eq4} \\
 &\leq& \eta\sqrt{\frac{n}{\delta}}\|\*y-\*u(t)\| \nonumber
\end{eqnarray}
Since ReLU is a $1$-Lipschitz function and $|\*v_r| = 1$,  we get \eqref{eq:thm1eq1}. We have \eqref{eq:thm1eq2} by applying Cauchy-Schwartz on the sum over $|\*S_{i}^{\perp}|$ terms. We get \eqref{eq:thm1eq3} due to Holder inequality. The $\*A_{\*S_{i}^{\perp}}$ is a $m \times d$ dimensional matrix. Here every row of this matrix is either all zero vector or a row from $\*A$. For all $j \in [m]$, if $j \in \*S_{i}^{\perp}$, then the $j^{th}$ row of $\*A_{\*S_{i}^{\perp}}$ is equal to $\*A_{j}$, else it is $\{0\}^{k}$. So in \eqref{eq:thm1eq4} we upper bound $\|\*A_{\*S_{i}^{\perp}}\| \leq 1$. Further, $\|\*A\|^{2} \leq \frac{m}{k}$ and $\forall j \in [n], \|\*Z_{j}(t)\| \leq \sqrt{m}$. 

Now we express  $I^{i}_{1}$ as follows,
\begin{eqnarray}
 I^{i}_{1} &=& -\frac{\eta}{m}\sum_{j \leq n}(\*y_{j} - \*u_{j}(t))(\*X_{i}\*C^{T}\*C\*X_{j}^{T})\left(\sum_{r \in \*S_{i}}(\*Z_{i}(t)\*A_{r}^{T}\*A_{r}\*Z_{j}^{T}(t))\right) \nonumber \\
 &=& -\eta\sum_{j \leq n}(\*y_{j} - \*u_{j}(t))(\*H_{i,j}(t) - \*H_{i,j}^{\perp}(t)) \nonumber
\end{eqnarray}

Here $\*H_{i,j}(t) = \*X_{i}\*C^{T}\*C\*X_{j}^{T}(\*Z_{i}(t)\*A\*A^{T}\*Z_{j}^{T}(t))$ and the term $\*H_{i,j}^{\perp}(t) = \*X_{i}\*C^{T}\*C\*X_{j}^{T}(\*Z_{i}(t)\*A_{\*S_{i}^{\perp}}\*A_{\*S_{i}^{\perp}}^{T}\*Z_{j}^{T}(t))$. $\*H^{\perp}$ is an $n \times n$ PSD matrix whose spectral norm can be bounded as follows,
\begin{eqnarray}
  \|\*H^{\perp}\| &\leq& \sum_{i,j}|\*H_{i,j}^{\perp}| = \frac{1}{m}\sum_{i \leq n} \left(\sum_{j \leq n} \left|\*X_{i}\*C^{T}\*C\*X_{j}^{T}(\*Z_{i}(t)\*A_{\*S_{i}^{\perp}}\*A_{\*S_{i}^{\perp}}^{T}\*Z_{j}^{T}(t))\right|\right) \nonumber \\
  &\leq& \frac{1}{m}\sum_{i \leq n} \left(\sum_{j \leq n} \left|O(1) \cdot (\*Z_{i}(t)\*A_{\*S_{i}^{\perp}}\*A_{\*S_{i}^{\perp}}^{T}\*Z_{j}^{T}(t))\right|\right) \nonumber \\
  &=& \frac{1}{m}\sum_{i \leq n} \left(\sum_{j \leq n} \left|O(1) \cdot (\tilde{\*Z}_{i}(t)\*A_{\*S_{i}^{\perp}}\*A_{\*S_{i}^{\perp}}^{T}\tilde{\*Z}_{j}^{T}(t))\right|\right) \label{eq1} \\
  &\leq& \frac{1}{m}\sum_{i \leq n} \left(\sum_{j \leq n} \left|O(1) \cdot \|\*A_{\*S_{i}^{\perp}}\tilde{\*Z}_{i}^{T}(t)\|\|\*A_{\*S_{i}^{\perp}}\tilde{\*Z}_{j}^{T}(t)\|\right|\right) \nonumber \\
  &\leq& \frac{1}{m}\sum_{i \leq n} \left(\sum_{j \leq n} \left|O(1) \cdot \|\*A\tilde{\*Z}_{i}^{T}(t)\|\|\*A\tilde{\*Z}_{j}^{T}(t)\|\right|\right) \label{eq2} \\
  &\leq& \frac{1}{m}\sum_{i \leq n} \left(\sum_{j \leq n} O(1) \frac{mR}{\delta}\right) \nonumber \leq O\left(\frac{n^{2}R}{\delta}\right) \nonumber
\end{eqnarray}

In \eqref{eq1} the vectors $\tilde{\*Z}_{i}(t)$ and $\tilde{\*Z}_{j}(t)$ are $m$ dimensional sparse vectors. For every index of $r \in \*S_{i}^{\perp}$, $\tilde{\*Z}_{i,r}(t) = \*Z_{i,r}(t)$. Similarly the vector $\tilde{\*Z}_{j}(t)$ is also defined. So their sparsity is bounded by $|\*S_{i}^{\perp}|$. 

 $\tilde{\*Z}_{i}(t)$ and $\tilde{\*Z}_{j}(t)$. In \eqref{eq2} we use the upper bound $\|\*A_{\*S_{i}^{\perp}}\*Y^{T}\| \leq \|\*A\*Y^{T}\|$, where $\*Y \in \~R^{m}$.

Finally we bound our desired term,
\begin{eqnarray}
 \|\*y - \*u(t+1)\|^{2} &=& \|\*y - \*u(t) - (\*u(t+1) - \*u(t))\|^{2} \nonumber \\
 &=& \|\*y-\*u(t)\|^{2} + \|\*u(t+1) - \*u(t)\|^{2} - 2(\*y-\*u(t))^{T}(\*u(t+1) - \*u(t))\nonumber \\
 &=& \|\*y-\*u(t)\|^{2} + \|\*u(t+1) - \*u(t)\|^{2} - 2(\*y-\*u(t))^{T}\*I_{2}\nonumber \\
 &-& 2\eta(\*y-\*u(t))^{T}\*H(t)(\*y-\*u(t)) + 2\eta(\*y-\*u(t))^{T}\*H(t)^{\perp}(\*y-\*u(t)) \nonumber \\
 &\leq& \left(1 + \frac{\eta^{2}n^{2}}{\delta^{4}} - \sum_{i \leq n}2\eta\sqrt{\frac{n}{\delta}} - \eta\lambda_{0} + \frac{2\eta n^{2}R}{\delta}\right)\|\*y - \*u(t)\|^{2} \nonumber \\
 &\leq& \left(1 - \frac{\eta\lambda_{0}}{2}\right)\|\*y - \*u(t)\|^{2} \label{eq3} \\
 &\leq& \left(1 - \frac{\eta\lambda_{0}}{2}\right)^{t}\|\*y - \*u(0)\|^{2} \nonumber 
\end{eqnarray}

By setting $\eta = \frac{\lambda_{0}\delta^{4}}{2n^{2}}$ we get \eqref{eq3}.
\end{proof}

\paragraph{Running Time:} The previous theorem essentially implies that in an overparameterized network with network width of $O\left(\frac{n^6}{\lambda_{0}^4\delta^7}\right)$ one can find global minimum by gradient descent with a learning rate $\eta = O\left(\frac{\lambda_{0}\delta^{4}}{n^{2}}\right)$. However, notice that due to our dimensionality reduction, our network can achieve it much faster even for network width of $O(n^{6})$. In each iteration we update a $k \times \ell$ dimensional matrix matrix $\*B$. For all $i \in [n]$ we do a matrix vector product, $\*Z_{i}(t)\*A$ and $\*X_{i}\*C^{T}$ which takes $O(mk + \ell d) = O(mk)$. Now instead of using $\*A$ as naive dimensionality reduction matrix (i.e., Johnson–Lindenstrauss), if replace it with Fast Johnson–Lindenstrauss transformation \cite{ailon2006approximate}, then our running time improves to $O(m\log(k)) = O\left(n^{6}\log\left(\frac{n}{\lambda_{0}\delta}\right)\right)$. As in each iteration of the gradient descent process uses all the $n$ inputs, hence updating $\*B(t)$ takes $O(nm\log(k)) = O\left(n^{7}\log\left(\frac{n}{\lambda_{0}\delta}\right)\right)$. Now for $t \geq \frac{\log(\varepsilon)}{\log(1-\eta\lambda_{0})}$ our learnt parameter ensures,
\begin{equation*}
  \|\*y - \*u(t+1)\|^{2} \leq \varepsilon\cdot\|\*y - \*u(0)\|^{2}.
\end{equation*}

So the total running time to achieve this is $O\left(\frac{n^{7}\log\left(\frac{n}{\lambda_{0}\delta}\right)\log(\varepsilon)}{\log(1-\eta\lambda_{0})}\right)$. Notice that with existing technique the running time would be $O(nmdt)$ which is $O\left(\frac{n^{7}d\log(\varepsilon)}{\lambda_{0}^{4}\delta^{2}\log(1-(\lambda_{0}/n)^{2})}\right)$.



%% file: sections/ack.tex
\section{Conclusion}
In this work we presented a simple yet powerful way to reduce the required number of unknown parameters in an overparameterized network. 
Our result uses a dimentionality reduction approach for input points as well as for the vector fed into the neurons of the latent layer. We get a significant improvement in the running time while retaining the original theoretical NTK based guarantees. In particular, when $d = O(n)$ our network achieves a  similar training loss with a faster running time by factor $O(n)$ compared to other related results. 

\section{Acknowledgements}
This project has received funding from the European Union’s Horizon 2020 research and innovation programmed under grant agreement No 682203 -ERC-[ Inf-Speed-Tradeoff].

%% file: main.bbl
\begin{thebibliography}{10}

\bibitem{johnson1984extensions}
William~B Johnson and Joram Lindenstrauss.
\newblock Extensions of lipschitz mappings into a hilbert space 26.
\newblock {\em Contemporary mathematics}, 26:28, 1984.

\bibitem{lecun1998gradient}
Yann LeCun, L{\'e}on Bottou, Yoshua Bengio, and Patrick Haffner.
\newblock Gradient-based learning applied to document recognition.
\newblock {\em Proceedings of the IEEE}, 86(11):2278--2324, 1998.

\bibitem{ailon2006approximate}
Nir Ailon and Bernard Chazelle.
\newblock Approximate nearest neighbors and the fast johnson-lindenstrauss
  transform.
\newblock In {\em Proceedings of the thirty-eighth annual ACM symposium on
  Theory of computing}, pages 557--563, 2006.

\bibitem{matouvsek2008variants}
Ji{\v{r}}{\'\i} Matou{\v{s}}ek.
\newblock On variants of the johnson--lindenstrauss lemma.
\newblock {\em Random Structures \& Algorithms}, 33(2):142--156, 2008.

\bibitem{mahoney2009cur}
Michael~W Mahoney and Petros Drineas.
\newblock Cur matrix decompositions for improved data analysis.
\newblock {\em Proceedings of the National Academy of Sciences},
  106(3):697--702, 2009.

\bibitem{ailon2010faster}
Nir Ailon and Bernard Chazelle.
\newblock Faster dimension reduction.
\newblock {\em Communications of the ACM}, 53(2):97--104, 2010.

\bibitem{collobert2011natural}
Ronan Collobert, Jason Weston, L{\'e}on Bottou, Michael Karlen, Koray
  Kavukcuoglu, and Pavel Kuksa.
\newblock Natural language processing (almost) from scratch.
\newblock {\em Journal of machine learning research}, 12(ARTICLE):2493--2537,
  2011.

\bibitem{drineas2011faster}
Petros Drineas, Michael~W Mahoney, Shan Muthukrishnan, and Tam{\'a}s
  Sarl{\'o}s.
\newblock Faster least squares approximation.
\newblock {\em Numerische mathematik}, 117(2):219--249, 2011.

\bibitem{drineas2012fast}
Petros Drineas, Malik Magdon-Ismail, Michael~W Mahoney, and David~P Woodruff.
\newblock Fast approximation of matrix coherence and statistical leverage.
\newblock {\em The Journal of Machine Learning Research}, 13(1):3475--3506,
  2012.

\bibitem{hinton2012deep}
Geoffrey Hinton, Li~Deng, Dong Yu, George~E Dahl, Abdel-rahman Mohamed, Navdeep
  Jaitly, Andrew Senior, Vincent Vanhoucke, Patrick Nguyen, Tara~N Sainath,
  et~al.
\newblock Deep neural networks for acoustic modeling in speech recognition: The
  shared views of four research groups.
\newblock {\em IEEE Signal processing magazine}, 29(6):82--97, 2012.

\bibitem{krizhevsky2012imagenet}
Alex Krizhevsky, Ilya Sutskever, and Geoffrey~E Hinton.
\newblock Imagenet classification with deep convolutional neural networks.
\newblock {\em Advances in neural information processing systems}, 25, 2012.

\bibitem{liberty2013simple}
Edo Liberty.
\newblock Simple and deterministic matrix sketching.
\newblock In {\em Proceedings of the 19th ACM SIGKDD international conference
  on Knowledge discovery and data mining}, pages 581--588, 2013.

\bibitem{woodruff2014sketching}
David~P Woodruff et~al.
\newblock Sketching as a tool for numerical linear algebra.
\newblock {\em Foundations and Trends{\textregistered} in Theoretical Computer
  Science}, 10(1--2):1--157, 2014.

\bibitem{larsen2017optimality}
Kasper~Green Larsen and Jelani Nelson.
\newblock Optimality of the johnson-lindenstrauss lemma.
\newblock In {\em 2017 IEEE 58th Annual Symposium on Foundations of Computer
  Science (FOCS)}, pages 633--638. IEEE, 2017.

\bibitem{baykal2018data}
Cenk Baykal, Lucas Liebenwein, Igor Gilitschenski, Dan Feldman, and Daniela
  Rus.
\newblock Data-dependent coresets for compressing neural networks with
  applications to generalization bounds.
\newblock In {\em International Conference on Learning Representations}, 2018.

\bibitem{du2018gradient}
Simon~S Du, Xiyu Zhai, Barnabas Poczos, and Aarti Singh.
\newblock Gradient descent provably optimizes over-parameterized neural
  networks.
\newblock {\em arXiv preprint arXiv:1810.02054}, 2018.

\bibitem{jacot2018neural}
Arthur Jacot, Franck Gabriel, and Cl{\'e}ment Hongler.
\newblock Neural tangent kernel: Convergence and generalization in neural
  networks.
\newblock {\em Advances in neural information processing systems}, 31, 2018.

\bibitem{arora2019fine}
Sanjeev Arora, Simon Du, Wei Hu, Zhiyuan Li, and Ruosong Wang.
\newblock Fine-grained analysis of optimization and generalization for
  overparameterized two-layer neural networks.
\newblock In {\em International Conference on Machine Learning}, pages
  322--332. PMLR, 2019.

\bibitem{arora2019exact}
Sanjeev Arora, Simon~S Du, Wei Hu, Zhiyuan Li, Russ~R Salakhutdinov, and
  Ruosong Wang.
\newblock On exact computation with an infinitely wide neural net.
\newblock {\em Advances in Neural Information Processing Systems}, 32, 2019.

\bibitem{dao2019learning}
Tri Dao, Albert Gu, Matthew Eichhorn, Atri Rudra, and Christopher R{\'e}.
\newblock Learning fast algorithms for linear transforms using butterfly
  factorizations.
\newblock In {\em International conference on machine learning}, pages
  1517--1527. PMLR, 2019.

\bibitem{devlin2019bert}
Jacob Devlin, Ming-Wei Chang, Kenton Lee, and Kristina Toutanova.
\newblock Bert: Pre-training of deep bidirectional transformers for language
  understanding.
\newblock In {\em Proceedings of the 2019 Conference of the North American
  Chapter of the Association for Computational Linguistics: Human Language
  Technologies, Volume 1 (Long and Short Papers)}, pages 4171--4186, 2019.

\bibitem{mussay2019data}
Ben Mussay, Margarita Osadchy, Vladimir Braverman, Samson Zhou, and Dan
  Feldman.
\newblock Data-independent neural pruning via coresets.
\newblock In {\em International Conference on Learning Representations}, 2019.

\bibitem{strubell2019energy}
Emma Strubell, Ananya Ganesh, and Andrew McCallum.
\newblock Energy and policy considerations for deep learning in nlp.
\newblock In {\em Proceedings of the 57th Annual Meeting of the Association for
  Computational Linguistics}, pages 3645--3650, 2019.

\bibitem{alemohammad2020recurrent}
Sina Alemohammad, Zichao Wang, Randall Balestriero, and Richard Baraniuk.
\newblock The recurrent neural tangent kernel.
\newblock {\em arXiv preprint arXiv:2006.10246}, 2020.

\bibitem{chhaya2020streaming}
Rachit Chhaya, Jayesh Choudhari, Anirban Dasgupta, and Supratim Shit.
\newblock Streaming coresets for symmetric tensor factorization.
\newblock In {\em International Conference on Machine Learning}, pages
  1855--1865. PMLR, 2020.

\bibitem{chhaya2020coresets}
Rachit Chhaya, Anirban Dasgupta, and Supratim Shit.
\newblock On coresets for regularized regression.
\newblock In {\em International Conference on Machine Learning}, pages
  1866--1876. PMLR, 2020.

\bibitem{huang2020coresets}
Lingxiao Huang and Nisheeth~K Vishnoi.
\newblock Coresets for clustering in euclidean spaces: importance sampling is
  nearly optimal.
\newblock In {\em Proceedings of the 52nd Annual ACM SIGACT Symposium on Theory
  of Computing}, pages 1416--1429, 2020.

\bibitem{khodak2020initialization}
Mikhail Khodak, Neil~A Tenenholtz, Lester Mackey, and Nicolo Fusi.
\newblock Initialization and regularization of factorized neural layers.
\newblock In {\em International Conference on Learning Representations}, 2020.

\bibitem{mirzasoleiman2020coresets}
Baharan Mirzasoleiman, Kaidi Cao, and Jure Leskovec.
\newblock Coresets for robust training of neural networks against noisy labels.
\newblock {\em Neural Information Processing Systems (NeurIPS)}, 2020.

\bibitem{tukan2020coresets}
Morad Tukan, Alaa Maalouf, and Dan Feldman.
\newblock Coresets for near-convex functions.
\newblock {\em Advances in Neural Information Processing Systems}, 33, 2020.

\bibitem{ailon2021sparse}
Nir Ailon, Omer Leibovitch, and Vineet Nair.
\newblock Sparse linear networks with a fixed butterfly structure: theory and
  practice.
\newblock In {\em Uncertainty in Artificial Intelligence}, pages 1174--1184.
  PMLR, 2021.

\bibitem{han2021random}
Insu Han, Haim Avron, Neta Shoham, Chaewon Kim, and Jinwoo Shin.
\newblock Random features for the neural tangent kernel.
\newblock {\em arXiv preprint arXiv:2104.01351}, 2021.

\bibitem{jiang2021coresets}
Shaofeng Jiang, Robert Krauthgamer, Xuan Wu, et~al.
\newblock Coresets for clustering with missing values.
\newblock {\em Advances in Neural Information Processing Systems}, 34, 2021.

\bibitem{mai2021coresets}
Tung Mai, Cameron Musco, and Anup Rao.
\newblock Coresets for classification--simplified and strengthened.
\newblock {\em Advances in Neural Information Processing Systems},
  34:11643--11654, 2021.

\bibitem{munteanu2021oblivious}
Alexander Munteanu, Simon Omlor, and David Woodruff.
\newblock Oblivious sketching for logistic regression.
\newblock In {\em International Conference on Machine Learning}, pages
  7861--7871. PMLR, 2021.

\bibitem{chhaya2022coresets}
Rachit Chhaya, Anirban Dasgupta, Jayesh Choudhari, and Supratim Shit.
\newblock On coresets for fair regression and individually fair clustering.
\newblock In {\em International Conference on Artificial Intelligence and
  Statistics}, pages 9603--9625. PMLR, 2022.

\bibitem{shitonline}
Supratim Shit, Anirban Dasgupta, Rachit Chhaya, and Jayesh Choudhari.
\newblock Online coresets for parameteric and non-parametric bregman
  clustering.
\newblock {\em TMLR}, 2022.

\bibitem{tukan2022new}
Murad Tukan, Xuan Wu, Samson Zhou, Vladimir Braverman, and Dan Feldman.
\newblock New coresets for projective clustering and applications.
\newblock In {\em International Conference on Artificial Intelligence and
  Statistics}, pages 5391--5415. PMLR, 2022.

\end{thebibliography}
